\documentclass[final,onefignum,onetabnum]{siuro210301}

\usepackage{lipsum}
\usepackage{amsfonts}
\usepackage{graphicx}
\usepackage{algorithmic}
\usepackage{array} 
\usepackage{float} 
\usepackage{caption} 
\usepackage{colortbl} 
\usepackage{mathtools}
\usepackage{comment}
\usepackage{multirow}  
\usepackage{graphicx}  

\ifpdf
  \DeclareGraphicsExtensions{.eps,.pdf,.png,.jpg}
\else
  \DeclareGraphicsExtensions{.eps}
\fi

\usepackage{enumitem}
\setlist[enumerate]{leftmargin=.5in}
\setlist[itemize]{leftmargin=.5in}

\newsiamremark{remark}{Remark}
\newsiamremark{hypothesis}{Hypothesis}
\crefname{hypothesis}{Hypothesis}{Hypotheses}
\newsiamthm{claim}{Claim}

\headers{Constructing the Ideal Building Blocks Set}{R. Salazar}

\title{Letters, Colors, and Words: Constructing the Ideal Building Blocks Set.}

\author{Ricardo Esteban Salazar Ordoñez\thanks{Lake Forest College., Lake Forest, IL 
  (\email{salazarordon@lakeforest.edu}).} }

\dedication{\small\textit{Project advisor: Shahrzad Jamshidi\thanks{Lake Forest College, Lake Forest, IL (\email{sjamshidi@lakeforest.edu}).}}}

\usepackage{amsopn}

\ifpdf
\hypersetup{
  pdftitle={Letters, Colors, and Words: Constructing the Ideal Building Blocks Set.},
  pdfauthor={R. Salazar}
}
\fi

\begin{document}

\maketitle

\begin{abstract}
Define {\bf a building blocks set} to be a collection of $n$ cubes (each with six sides) where each side is assigned one letter and one color from a palette of $m$ colors. We propose a novel problem of assigning letters and colors to each face so as to maximize the number of words one can spell from a chosen dataset that are either {\bf mono words}, all letters have the same color, or {\bf rainbow words}, all letters have unique colors. We explore this problem considering a chosen set of English words, up to six letters long, from a typical vocabulary of a US American 14 year old and explore the problem when $n=6$ and $m=6$, with the added restriction that each color appears exactly once on the cube. The problem is intractable, as the size of the solution space makes a brute force approach computationally infeasible. Therefore we aim to solve this problem using random search, simulated annealing, two distinct tree search approaches (greedy and best-first), and a genetic algorithm. To address this, we explore a range of optimization techniques: random search, simulated annealing, two distinct tree search methods (greedy and best-first), and a genetic algorithm. Additionally, we attempted to implement a reinforcement learning approach; however, the model failed to converge to viable solutions within the problem's constraints. Among these methods, the genetic algorithm delivered the best performance, achieving a total of 2846 mono and rainbow words.

\end{abstract}

\section{Introduction}\label{sec:intro}
Building blocks are an iconic toy designed to be employed in block play, whose learning outcomes (if employed since early ages) range from development of motor, and classification, to mathematical and problem-solving skills \cite{blockPlay}. Our notion of a building blocks set extends this concept focusing on the integration language learning with color recognition to create a constrained problem for children: what words can you spell that are uniform in color ({\bf mono words}) or consisting of all distinct colors ({\bf rainbow words}). Formally, we define a {\bf building block set} to correspond to $n$ cubes. For each face of the cube, we assign one letter and one of color from a specified palette. In this paper, we focus on the case where $n=6$ and suppose our palette consists of 6 colors. In addition, we restrict the coloring so that each face of the cube has a unique color. \cref{fig:1} shows the template for one of the described cubes.

\begin{figure}[htbp]
  \centering
  \includegraphics{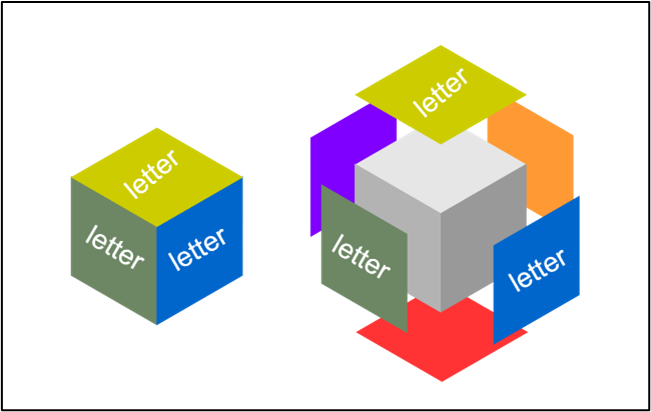}
  \caption{Cube colouring template.}
  \label{fig:1}
\end{figure}

We intend to find a set of cubes that can be used as a spelling game where children aim to find as many mono and rainbow English words as they can. See \cref{fig:2} for an example of each type of word.

\begin{figure}[htbp]
  \centering
  \includegraphics[width=\textwidth]{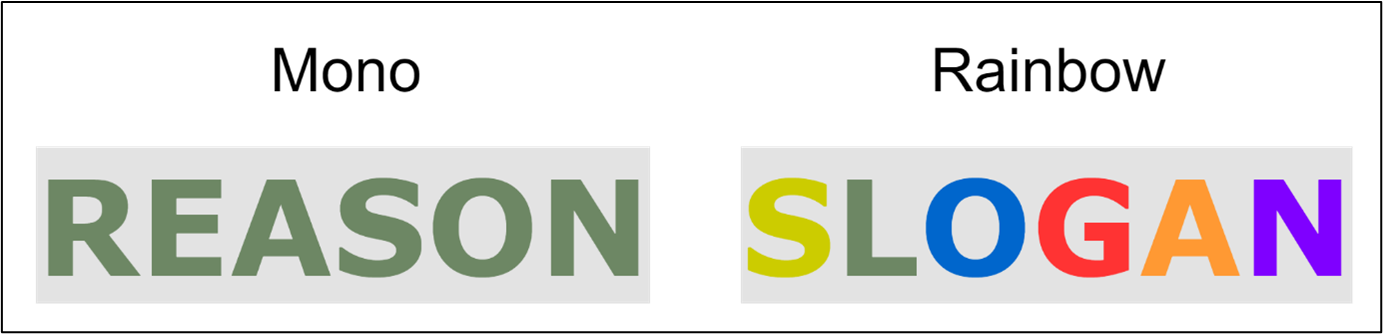}
  \caption{Example of Mono and Rainbow words.}
  \label{fig:2}
\end{figure}

In \cref{sec:set_up} we reduced the original dataset to only words of up to six letters, represented our cube set as a 36-item array, and decided what letters would be repeated after we make sure that the 26 letters of the alphabet are included.
In \cref{subsec:RS} we look at random search, method that after 1.3 million permutations its best solution could spell 1881 words. Consequently in \cref{subsec:SA} we implemented simulated annealing as a way to make our search more structured, which allowed us to land to a solution of 2352 words after 250 thousand iterations. In \cref{subsec:tree} we constrained the search a bit more by structuring a tree and first implementing a constrained greedy algorithm, which yielded a solution with a total of 2268 words. We then implemented best-first search algorithm which found a permutation that allows for 2323 words. Then tried a regular greedy algorithm, using this method, we were able to find a permutation with a rating of 2386 words. In \cref{subsec:RL} we attempted to create an RL agent that would navigate the tree structure (same set up as the previous section), however the the training never converged so we decided to move on and in \cref{subsec:genetic} where we try to find a solution through a genetic algorithm which allowed us to find a permutation with a total of 2846 words.

\section{Set up}\label{sec:set_up}

\subsection{Data set}\label{subsec:data_set}

Given that we are focusing on making a puzzle toy to spell English words with children as target users. Our data set consists of words ranked by Age-of-Acquisition from zero to 14 years old. The original list of words had approximately 44 thousand words \cite{AoA} , between duplicates and acronyms. For example, `stock' appeared 14 times to account for the different meanings. These were removed. However, since we have only six cubes, we restricted the dataset to words that are 6 characters long or less, leaving us with a total of 9624 words.  

\subsection{One-dimensional representation of a cube set}\label{subsec:1D}
Since the orientation of the faces on the cube is unimportant, we represent a single cube as a six-element array of letters where the indices act as labels for the six colors (see \cref{fig:3}).The six block set is then represented as a 36-item array where the color label is computed using modular arithmetic. For example, the color at index 11 $\coloneq  11 \mod 6$ (see \cref{fig:4}). Additionally, the cube label can be computed by using the floor function. For example, the letter at index 32 belongs to the cube $\lfloor \dfrac{32}{6} \rfloor$ (see \cref{fig:4})
\begin{figure}[htbp]
  \centering
  \includegraphics[width=\textwidth]{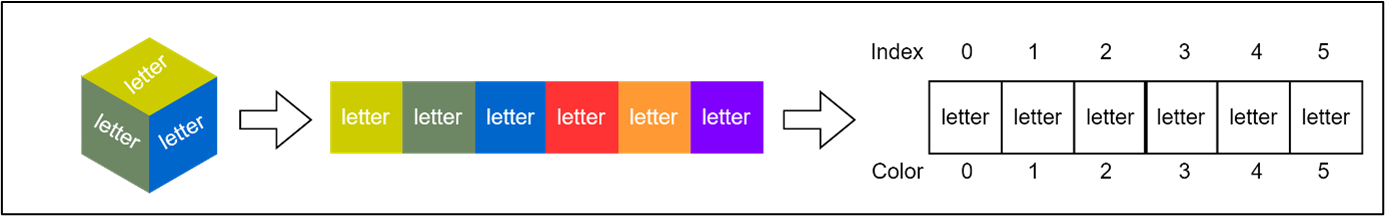}
  \caption{Representation of a cube in one dimension.}
  \label{fig:3}
  \includegraphics[width=\textwidth]{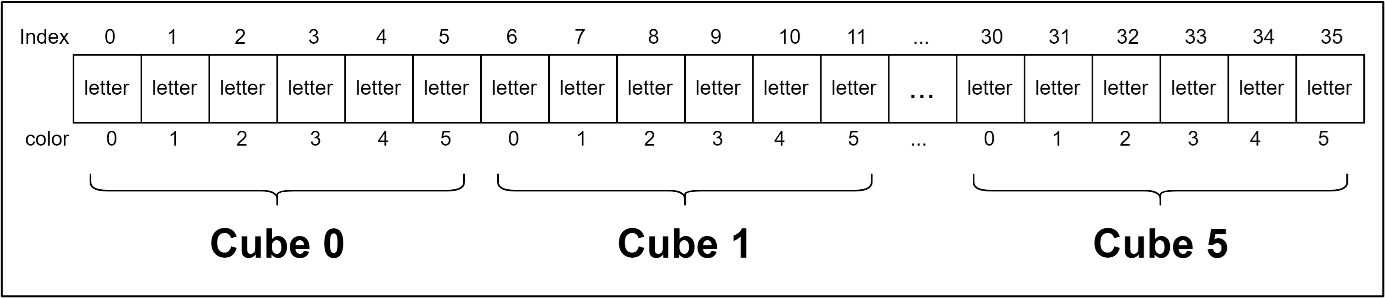}
  \caption{Representation of a six-block set in one dimension.}
  \label{fig:4}
\end{figure}

\subsection{Letter repetition}\label{subsec:letter_rep}
As we want to make sure to include each letter of the English alphabet at least once, we have ten vacant spots in our cube-set. This remainder could be distributed by assigning more repetitions to the most frequent letter. That is why we calculated the each letter frequency relative to our data set. \cref{tab:letter_freq} shows the letters sorted by frequency, we decided to round the frequency to the thousandth place so that every frequency is non-zero. We use this value to calculate the percentage of the remainder (rounded) that should be allocated to a given letter, starting from the most frequent until the corresponding portion of a letter is equal to zero, then this process is repeated using an updated number of reminder spots, until all the spots had been taken, or the corresponding portion for the most frequent letter is zero, in which case we simply assign whatever number of spots are left, to it. In other words, every letter that has a relative frequency $\geq 0.05$ gets assigned an extra repetition. Then the new number of vacant spots is 1, and we simply allocate it to letter `e', thus completing the 36 letters in our six-cube set. \cref{tab:letter_rep} shows the number of repetitions for our set (the highlighted column), we also include the letter repetitions that other cube sets may have if we use the same strategy for distributing the vacant spots. Using this information we created an array of 36-letters (see \cref{tab:base}) that yields a word count of 960, and we label as `The base permutation', since we will use it as a starting point whenever we create a random permutation. \cref{fig:5} illustrates the full one-dimensional representation of a cube set that uses `The base permutation'.

\begin{table}[htbp]
    \centering
    \begin{minipage}{0.45\textwidth}
        \centering
        \caption{Letters sorted by frequency}\label{tab:letter_freq}
        \begin{tabular}{|c|c|}
            \hline
            Letter & Relative frequency\\
            \hline
            e & 0.115 \\
        		a & 0.086 \\
        		r & 0.073 \\
        		o & 0.066 \\
        		t & 0.063 \\
        		i & 0.061 \\
        		s & 0.059 \\
        		l & 0.059 \\
        		n & 0.054 \\
        		u & 0.040 \\
        		d & 0.037 \\
        		c & 0.037 \\
        		p & 0.033 \\
        		m & 0.031 \\
       		h & 0.029 \\
        		y & 0.027 \\
        		g & 0.027 \\
        		b & 0.026 \\
        		f & 0.019 \\
        		w & 0.016 \\
        		k & 0.016 \\
        		v & 0.012 \\
        		z & 0.004 \\
        		x & 0.004 \\
        		j & 0.003 \\
        		q & 0.002 \\
            \hline
        \end{tabular}
    \end{minipage}%
    \hfill 
    \begin{minipage}{0.45\textwidth}
        \centering
        \caption{Letter Repetitions}\label{tab:letter_rep}
        \begin{tabular}{|c|c|c|c|c|c|c|}
        \cline{2-7}
        		  \multicolumn{1}{c|}{}& \multicolumn{6}{c|}{Number of cubes} \\
        \hline
        		Letter & 6 & 10 & 15 & 20 & 25 & 30 \\
        \hline
        e & \cellcolor{pink}3 & 6 & 9 & 12 & 16 & 20 \\
        a & \cellcolor{pink}2 & 4 & 6 & 9 & 13 & 14 \\
        r & \cellcolor{pink}2 & 3 & 6 & 8 & 10 & 12 \\
        o & \cellcolor{pink}2 & 3 & 5 & 7 & 9 & 11 \\
        t & \cellcolor{pink}2 & 3 & 5 & 7 & 9 & 11 \\
        i & \cellcolor{pink}2 & 3 & 5 & 7 & 9 & 10 \\
        s & \cellcolor{pink}2 & 3 & 5 & 7 & 8 & 10 \\
        l & \cellcolor{pink}2 & 3 & 5 & 7 & 8 & 10 \\
        n & \cellcolor{pink}2 & 3 & 4 & 6 & 8 & 9 \\
        u & \cellcolor{pink}1 & 2 & 4 & 5 & 6 & 7 \\
        d & \cellcolor{pink}1 & 2 & 3 & 5 & 6 & 7 \\
        c & \cellcolor{pink}1 & 2 & 3 & 4 & 6 & 7 \\
        p & \cellcolor{pink}1 & 2 & 3 & 4 & 5 & 6 \\
        m & \cellcolor{pink}1 & 2 & 3 & 4 & 5 & 6 \\
        h & \cellcolor{pink}1 & 2 & 3 & 4 & 5 & 5 \\
        y & \cellcolor{pink}1 & 2 & 3 & 4 & 4 & 5 \\
        g & \cellcolor{pink}1 & 2 & 3 & 4 & 4 & 5 \\
        b & \cellcolor{pink}1 & 2 & 3 & 3 & 4 & 5 \\
        f & \cellcolor{pink}1 & 2 & 2 & 3 & 3 & 4 \\
        w & \cellcolor{pink}1 & 2 & 2 & 3 & 3 & 3 \\
        k & \cellcolor{pink}1 & 2 & 2 & 1 & 3 & 3 \\
        v & \cellcolor{pink}1 & 1 & 2 & 1 & 2 & 3 \\
        z & \cellcolor{pink}1 & 1 & 1 & 1 & 1 & 2 \\
        x & \cellcolor{pink}1 & 1 & 1 & 1 & 1 & 2 \\
        j & \cellcolor{pink}1 & 1 & 1 & 1 & 1 & 2 \\
        q & \cellcolor{pink}1 & 1 & 1 & 1 & 1 & 1 \\
        \hline
    \end{tabular}
    \end{minipage}
\end{table}

\begin{table}[htbp]
    \centering
    \caption{Base permutation}\label{tab:base}
    \scriptsize 
    \setlength{\tabcolsep}{3pt} 
    \begin{tabular}{|*{36}{c|}}
        \hline
        E&E& E & A & A & R & R & O & O & T & T & I & I & S & S & L & L & N & N & U & D & C & P & M & H & Y & G & B & F & W & K & V & Z & X & J &Q \\
        \hline
    \end{tabular}
\end{table}

\begin{figure}[htbp]
  \centering
  \includegraphics[width=\textwidth]{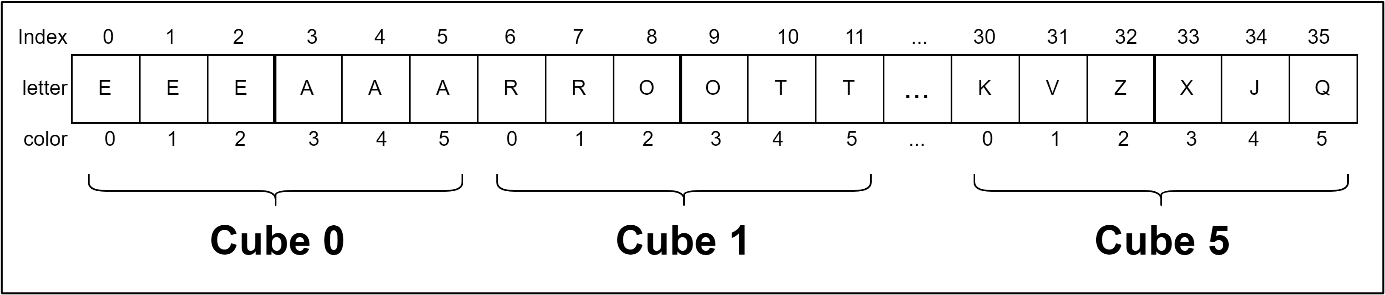}
  \caption{Base cube set.}
  \label{fig:5}
\end{figure}

\clearpage

\subsection{Search space}\label{subsec:search_space}
Given that our cube-set includes one letter repeated three times, and eight other repeated twice, the number of possible permutations is a follows:
\begin{equation}
	\dfrac{P_{36}^{36}}{3! (2!)^8} = 2.42183155*10^{38} \nonumber
\end{equation}

It takes approximately 1.07 seconds to count how many words a specific configuration can spell, so an exhaustive search will take about $6*10^{30} $ years. The problem is, therefore, intractable.

\section{Search algorithms}\label{sec:SA}
\subsection{Random Search}\label{subsec:RS}
To establish a baseline for comparison, we implemented a random search by shuffling the base permutation. We produced a total of 1.3 million random permutations and kept track of three maximization targets: mono words, rainbow words, and the sum of both. \cref{tab:RS_results} shows the results of the search showing the sum of mono and rainbow as `word count'. In \cref{app:RS} you can see the 10 best permutations found for each maximization target.

\begin{table}[htbp]
    \centering
    \caption{Results from Random search}\label{tab:RS_results}
    \begin{tabular}{|l|c|}
        \hline
        Maximization target & Word count \\
        \hline
        Mono  & 1569\\
        \hline
        Rainbow  & 1813\\
        \hline
        Sum  & 1881\\
        \hline
    \end{tabular}
\end{table}

\clearpage
\subsection{Simulated Annealing}\label{subsec:SA}
When thinking about algorithms that leverage the power of randomness, simulated annealing usually comes to mind. This search requires us to treat our problem as a graph, where we decide whether or not to move from one state to another (randomly chosen) neighbouring state if it's better than the current or based on a probability. Since the algorithm attempts to emulate a process of metallurgy, that probability is in function of a sense of high energy or `temperature' that will decrease slowly as time passes. So, the algorithm will initially be more likely to visit a state that performs worse than the current, during early stages in comparison to later stages. This is the way in which simulated annealing deals with the balance between exploration and exploitation \cite{simAnn}. In our case, a neighbour will be the result of swapping two items chosen at random. We ran this algorithm for 250 thousand iterations, with an initial `temperature' of 1000 and a cooling rate or 0.999. And the results are seen in \cref{tab:SA_results}.

\begin{table}[htbp]
    \centering
    \caption{Results from simulated annealing}\label{tab:SA_results}
    \begin{tabular}{|l|c|}
        \hline
        Maximization target & Word count \\
        \hline
        Mono  & 1497\\
        \hline
        Rainbow  & 2243\\
        \hline
        Sum  & 2352\\
        \hline
    \end{tabular}
\end{table}

\subsection{N-ary tree search algorithms}\label{subsec:tree}
In the interest of conducting a controlled exploration of the search space, we can create a uniform N-ary tree, where, every child node is the result of making a small modification in the parent. We found it suitable to make a 180-ary tree, so that the current node (parent) is connected to all unique face-swaps within a cube, for every cube ($6\binom{6}{2}$) along with all unique face-swaps for the letters of the same color, for every color ($6\binom{6}{2}$). For illustrative purposes, \cref{fig:6} shows an example of the fifteen unique element swaps within a cube, and \cref{fig:7} shows an example of the fifteen unique swaps within a color.

\begin{figure}[htpb]
  \centering
  \includegraphics{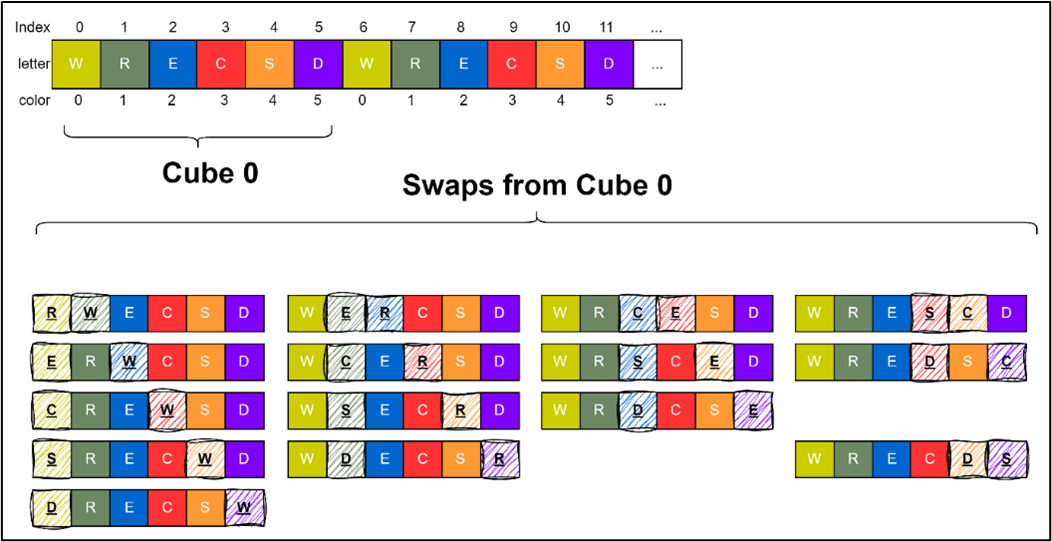}
  \caption{Example of swaps within a cube.}
  \label{fig:6}
  \includegraphics{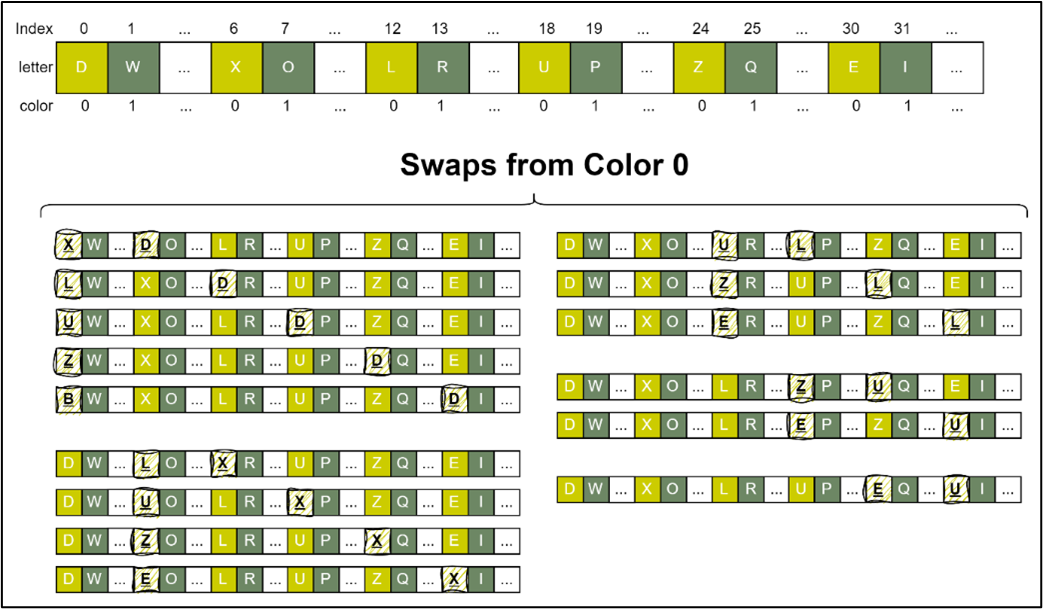}
  \caption{Example of swaps within a given color.}
  \label{fig:7}
\end{figure}

Although this method of creating child permutations allows for cyclic references and redundant states---due to some trivial swaps and the fact that some changes can be quickly reverted in just two generations---by keeping track of all visited permutations and excluding them as possible parents, we end up with a tree structure that constrains the search space enough to make greedy-like approaches computationally manageable, while simultaneously ensuring that any permutation in the search space is still theoretically reachable (see \cref{app:proof}).

It is important to note that because we explore this 180-ary tree using three greedy algorithms, we will only be able to find local maxima near the starting point. So, our success in finding a satisfactory result hinges on the selection of a diverse set of roots. We decided that we will test each case with three different roots: `Base' permutation, a random permutation ( as seen in \cref{tab:seed2K}) generated using a seed = 2000 which we label `Seed2K' and yields a word count of 1550, and, a third permutation, sourced from the results of \cref{subsec:RS}(E.g when maximizing mono, the root used will be the best permutation found through random search when we maximized mono words). 

\begin{table}[htbp]
    \centering
    \caption{Seed2K permutation}\label{tab:seed2K}
    \scriptsize 
    \setlength{\tabcolsep}{3pt} 
    \begin{tabular}{|*{36}{c|}}
        \hline
        E&S&D&A&U&O&Y&G&T&J&C&R&R&L&H&B&F&E&S&I&K&V&E&I&T&Q&N&X&W&M&L&A&Z&P&N&O\\
        \hline
    \end{tabular}
\end{table}
\clearpage

\subsubsection{Constrained greedy search}\label{subsec:CGA}
This algorithm will choose the parent of the next level by selecting a permutation that ranks equal or better than its precursor, if such a permutation exists, otherwise it will stop the search. Because of this constrain the last permutation to become a parent before the search stops, will be the best solution. We executed this algorithm with a budget of 130 thousand unique permutations per tree, and the results are shown in \cref{tab:CGA_results}.

\begin{table}[H]
    \centering
    \caption{Results from Greedy Algorithm}\label{tab:CGA_results}
    \begin{tabular}{|l|l|l|l|}
        \hline
        Target & Root & Word count & Found at \\
        \hline
        Mono  & Base 				  & 1509 & G484-P82115 \\
              & Seed2K	  & 1546 & G88-P15398 \\
              & Mono from random search & 1531 & G39-P6916 \\
        \hline
        Rainbow & \cellcolor{pink}Base & \cellcolor{pink}2268 & \cellcolor{pink}G25-P4444 \\
                & Seed2K   		& 2167 & G18-P3303 \\
                & Rainbow from random search & 2074 & G15-P2717 \\
        \hline
        Sum   & Base					  & 2116 & G16-P2888 \\
              & Seed2K	  & 2199 & G31-P5596 \\
              & Sum from random search & 2074 & G15-P2717 \\
        \hline
    \end{tabular}
\end{table}

The column `Found at' indicates us in which generation (E.g G15 is the 15th generation) and after how many unique permutations was the solution found (E.g P2717 means 2717th permutation). If we look carefully we can appreciate that several of the trees in \cref{tab:CGA_results} end up rather quickly, way before reaching the 130 thousand permutations.

\subsubsection{Best-first search}\label{subsec:BFS}
An approach that will help us make sure that we visit exactly the number of unique permutations that we want is best first search, that at every given level all of the child permutations are put into a priority queue. The parent of the next generation will simply be the permutation at the front of the queue. Because of the priority queue sorts its entries from highest to lowest word count, the next parent can come from any generation, not just from the immediate one. Different to \cref{subsec:CGA}, we need to keep track of and update the best permutation we come across.
After running nine trees with a target of at least 150 thousand permutations generated, the results are as show in \cref{tab:BFS_results}.

\begin{table}[H]
    \centering
    \caption{Results from Best First Search}\label{tab:BFS_results}
    \begin{tabular}{|l|l|l|l|}
        \hline
        Target & Root & Word count & Found at \\
        \hline
        Mono  & Base                   & 1949 & G236-P41396 \\
              & Seed2K      & 1858 & G335-P57992 \\
              & Mono from random search& 1867 & G28-P4948 \\
        \hline
        Rainbow & Base                 		& 2280 & G850-P148581 \\
                & Seed2K     		& 2234 & G744-P130448 \\
                & Rainbow from random search & 2258 & G603-P106256 \\
        \hline
        Sum   & Base                   &2296 & G671-P117117 \\
              & \cellcolor{pink}Seed2K       &\cellcolor{pink}2328 & \cellcolor{pink}G839-P146807 \\
              & Sum from random search &2323 & G483-P8494 \\
        \hline
    \end{tabular}
\end{table}
\subsubsection{Greedy search}\label{subsubsec:GA}
Thinking about how quickly \cref{subsec:CGA} would come to halt after a couple of generations, we decided to relaxed its constrains allowing the algorithm to run for as many iterations as specified. So, this algorithm will select the parent of the next generation by selecting the best child permutation in a generation even if such a candidate is not equal or better than its precursor. We will need to keep track of and update the best permutation we come across. We ran this set up for 130 thousand permutations. The results are listed in \cref{tab:GA_results}.

\begin{table}[H]
    \centering
    \caption{Results Greedy algorithm}\label{tab:GA_results}
    \begin{tabular}{|l|l|l|l|}
        \hline
        Target & Root & Word count & Found at \\
        \hline
        Mono  & Base                   &1509 & G484-P82115 \\
              &Seed2K     & 1546 & G88-P15398 \\
              & Mono from random search& 1531 & G39-P6916 \\
        \hline
        Rainbow & \cellcolor{pink}Base  & \cellcolor{pink}2386 & \cellcolor{pink}G652-P114941 \\
                & Seed2K    		& 2325 & G593-P104536 \\
                & Rainbow from random search & 2310 & G513-P90969 \\
        \hline
        Sum   & Base                   & 2286 & G439-P77233 \\
              & Seed2K      & 2360 & G277-P48954 \\
              & Sum from random search &2331 & G666-P117592 \\
        \hline
    \end{tabular}
\end{table}

\cref{tab:tree_algo_comparison} will help us compare the results from the three algorithms we tried in this sub section. As we can appreciate, the best result was obtained using the regular Greedy algorithm with the `Base' permutation as root of the tree.
\begin{table}[H]
    \centering
    \caption{Comparison for the tree algorithms}\label{tab:tree_algo_comparison}
    \begin{tabular}{|l|l|l|}
        \hline
        Algorithm & Root & Word count\\
        \hline
        Constrained greedy  &Base&2268\\
        \hline
        Best First Search & Seed2K&2328\\
        \hline
        \cellcolor{pink}Greedy & \cellcolor{pink}Base & \cellcolor{pink}2386\\
        \hline
    \end{tabular}
\end{table}

\subsection{Reinforcement Learning}\label{subsec:RL}
Given that several breakthroughs and advances thought to be near impossible to be conquered by machines has been achieved by Reinforcement learning (like AlphaGo or the multiple applications on video games and robotics)\cite{RL1} we thought that creating an RL agent and training it to make permutations out of our base permutation, would most likely give us good results. With the objective of making an RL model whose results can be comparable to our previous methods, we thought of an agent navigating the tree search structure described in \cref{subsec:tree}.The goal was for the agent to learn a strategy for selecting the parent of the next generation, aiming for the final state to be a high-ranking permutation. Since we won't be evaluating every child at each level, we can also describe the agent as one that builds a high-ranking permutation by making a series of valid swaps (see \cref{fig:6} and \cref{fig:7}).

\paragraph{Environment}
We used the Gymnasium API\cite{gym} by Farama Foundation to set up our Building Blocks environment. Gymnasium offers several pre-made RL environments to learn and experiment with, but also provides an interface to create custom environments.\\

\underline{Observation space:} Describing the observation space whether partial or fully observable. In our case, the observation space represents a procedure made of 36 sub procedures that are mutually independent, where each sub procedure can be done in 26 different ways. So, the cardinality of a set containing all the possible observations would be equal to$26^{36} $.  Following Gymnasium guidelines, we express this space as multi-discrete space of 36 items where every item can take up an integer between 0 to 25 ( Gymnasium.spaces.MultiDiscrete([26,26,26,…26]) ).\\

\underline{Action space:}  The 180 options from which the agent can take one at a given time step (The list of 2-item tuples that describe the Pitted swaps see image 6 and image 7). In Gymnasium it is described as a discrete space that can have a value between 0 and 179 ( \textit{Gymnasium.spaces.Discrete(180)}).\\

\underline{Reset Function:} That is called whenever the agent is initialized, or it reaches a terminal state. In the case of our custom environment, this function sets the current observation to a random permutation. It also initializes a visited list to keep track of the unique permutations that the agent has generated.\\

\underline{Step Function:} That is called at every time-step, it takes as input a valid action taken by the agent. We translate this action into a change in the current state by using the action chosen by the agent as an index of the `swaps list' and making the swap. We also calculate the reward and keep track of the current state in the visited list. The step function requires us to return a 4-item tuple containing an integer or decimal representing the reward, a Boolean describing whether the agent has reached a terminal state, a Boolean describing whether the agent has reached a truncated stated, and a dictionary with extra information about the action taken.\\

As you may have noticed, we talked about rewarding the agent in the `step' function call, but did not get into the details of how this was calculated. That is because there were two general ways in which we could reward: Immediate reward and delayed reward. With an immediate reward system, the agent would have more feedback on every specific action at a given state but would lack a sense of long-term benefits. And with a delayed reward, it has a sense of long terms benefits but being ignorant of how good specific actions are. We examined the two different reward systems, both of which require us to set up the training/testing process in terms of episodes of n time-steps each episode. For an immediate reward system, at every time-step, the reward will be a thousandth of the difference of the current state and the initial state's word count (which can be positive or negative), but only for those non-repeating permutations within an episode. Otherwise, the reward will be -1. With this penalty, we attempted to incentivize the agent to take actions that result in unique permutations. We also hoped to prevent loops where a change is made at time-step t and then reversed at t+1.
For a delayed reward system, the agent will only be rewarded at the end of an episode. The reward will be a thousandth of the difference between the current word count and the initial word count.\\

\paragraph{Algorithms}
When it comes to RL, StableBaselines3 \cite{sb3} has one of the most documented and reliable implementations of learning algorithms for python. Out of a list of 12 RL algorithms available through StableBaselines3, we tried two algorithms that work best for discrete spaces: Proximal Policy Optimization(PPO) and Deep Q-Network (DQN), both using a `multi-layer perceptron policy' (MlpPolicy).\\

\paragraph{Training}
We trained 4 agents, one using PPO and the immediate reward, one using DQN and immediate reward, one using PPO and delayed reward, and one using DQN and delayed reward. We ran each of them for a total of 1000 episodes with an episode length of 500 time-steps.\\

\paragraph{Evaluation}
When we want to evaluate a model trained using StableBaselines3 model, we have two options of prediction: deterministic and non-deterministic. However, for every one of the models evaluated with deterministic prediction, the agent would make a choice and then undo the change, resulting in an agent that simply oscillates between two choices and never progresses. That is why we relied on non-deterministic predictions. In \cref{tab:RL_results} you can find the best word count found after running each model for 50 episodes, where every episode lasts 500 time-steps.

\begin{table}[htbp]
    \centering
    \caption{Comparison of Immediate and Delayed Reward Training}\label{tab:RL_results}
    \begin{tabular}{|l|c|c|}
        \cline{2-3}
        \multicolumn{1}{c|}{}& PPO & DQN \\
        \hline
        Immediate reward training & 1853 words & 1737 words \\
        Delayed reward training   & 1804 words & 1725 words \\
        \hline
    \end{tabular}
\end{table}

Since RL agents learn by making random choices and then crating a policy based on the experiences it is exposed to, we created a `random agent' as a control case, that would make random choices at every time-step. We also ran this agent for 50 episodes, so that it could give us a better sense of how our agent improved after training.  The best word count found by our control agent was a total of \textbf{1832 words}. By contrast, our trained agents, showed no appreciable improvement, which tells us that it will require more training time for our agents to perform better than a random version of itself, and even more time for such agents to perform better than our previously executed methods. As a result, we abandoned this method and chose to explore genetic algorithms which are well suited to noisy problems.

\subsection{Genetic algorithm}\label{subsec:genetic}
genetic algorithms attempt to mimic the way genetic information is passed down from parents to offspring, through natural selection\cite{GA}. And since we are attempting to maximize the permutation of a 36-item array, a genetic algorithm seems promising as a sequence of letters is analogous to a sequence of genes, therefore in our context the individuals or chromosomes are our 36-item permutations.

\subsubsection{Set up} We will start from a specified number of random individuals, every generation after that is made up partly of elites from the previous generation, crossed, mutant and random individuals. We describe how these sub populations are generated hereafter.\\

\paragraph{Elites} The top number of specified individuals, that are transferred to the next generation intact.\\

\paragraph{Crossovers} Crossed offspring are created by combining different segments of genetic information ---also known as loci--- from two individuals. Although in biology the length of a locus varies, since we are working with a set of 6 cubes, it is beneficial for us to think of each cube as a locus. So, the process of crossover is a six-step process in which the corresponding locus is copied from either parent at random. See \cref{fig:crossover} for a visual representation. It is worth mentioning that this process of crossover will probably break the constraint that we stipulated in \cref{subsec:letter_rep}: "… to include each letter of the English alphabet at least once, ...". That is because the random choice of loci will result in the inclusion or exclusion of some letters, generating cube sets that are not a permutation of the `Base' permutation. We chose to have two types of crossovers: \underline{type 1} are the result of crossing two elites, and \underline{type 2} are the result of crossing an elite and a non-elite. The parents of both types are chosen randomly.
\begin{figure}[H]
  \centering
  \includegraphics{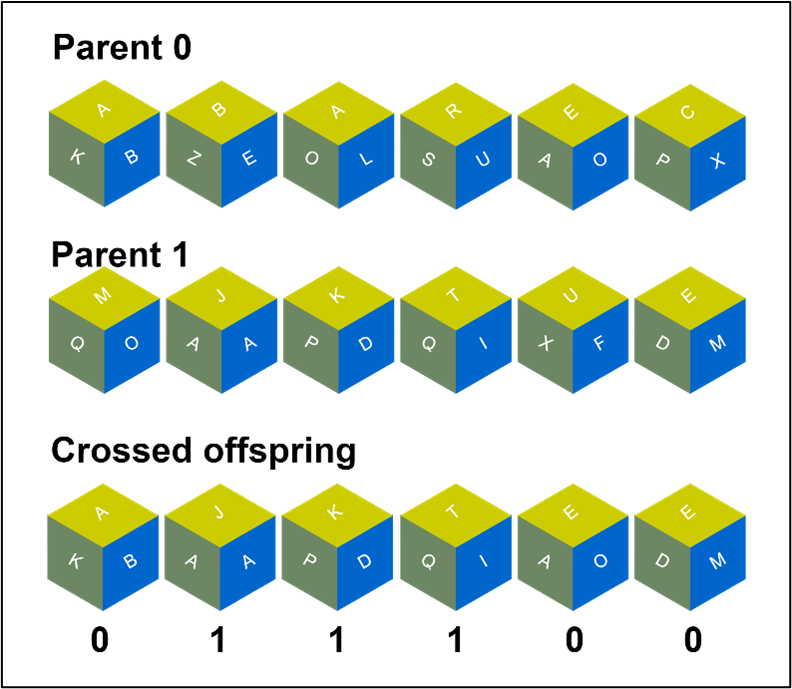}
  \caption{Example of a crossover}
  \label{fig:crossover}
\end{figure}

\paragraph{Mutations} Mutations happen by making a carbon copy of a given individual and simply swapping 2 of its items chosen at random. We will always mutate the best n elites.

\subsubsection{Tests}
In the interest to see if a `good' starting point will give a population an advantage, we created parallel scenarios for the starting point: The first scenario's original population is completely random, but for scenario 2 and 3, three individuals are arbitrarily chosen and the rest are randomly generated.\\
For scenario 2, we included the 'Base' permutation,  the 'Seed2K', and, the best permutation from \cref{subsec:RS}. The idea was for this scenario to have poor-performing permutations\\
For scenario 3 we included the best permutations found by each algorithm of \cref{subsec:tree} when maximizing for the sum of rainbow and mono words.

We ran 7 iterations with different parameters as shown in \cref{tab:genAl_results}. The `Found at :' row of each scenario tells us the generation in which a permutation with this word count was found for the first time (E.g G588 represents the 588th generation). At the beginning we tested 1000 generations per iteration, but as we see in iterations 1-3 the best results are found by the 323rd iteration (on average), so we decided to reduce the number of generations to 500.

\begin{table}[H]
    \centering
    \caption{Arguments and results for Iterations of genetic algorithm.}\label{tab:genAl_results}
    \begin{tabular}{|l|c|c|c|c|c|c|c|}
        \cline{2-8}
        \multicolumn{1}{c|}{}& \multicolumn{7}{c|}{Iteration}\\
        \cline{2-8}
        \multicolumn{1}{c|}{}& 1 & 2 &3 &4 &5 & 6 & 7 \\
        \hline
        Generations       & 1000 & 1000 & 1000 & 500 &500 & 500 & 500\\
        \hline
        Elites            & 40   & 20   & 40   & 20  & 20   & 10 & 20\\
        Crossed type 1    & 40   & 20   & 13   & 6 & 10   & 25 & 10\\
        Crossed type 2    & 0   & 0   & 27   & 14 & 10   & 25 & 10\\
        Mutated           & 10   & 30   & 10   & 40 & 40   & 20 & 50\\
        Random            & 10   & 30   & 10   & 20 & 20   & 20 & 10\\
        \hline
        Scenario 1 best:& 2717 & 2510 & 2775 & 2590 & 2726 & 2638 & 2810\\
        Found at	:       & G588 & G187 & G987 & G424 & G158 & G392 & G367\\
        	\hline
        Scenario 2 best:& 2706 & 2349 & 2698 & 2703 & 2705 & 2703 &2794\\
        Found at	:      & G485 & G64 &  G398 & G320 & G129 & G254 & G350\\
        \hline
        Scenario 3 best:& 2666 & 2360 & \cellcolor{pink}2846 & 2763 & 2749 & 2769 & 2778 \\
        Found at	:      & G77 &   G0 &  G845 & G260 & G82 & G316 & G362\\
        \hline
    \end{tabular}
\end{table}

The best permutation  was found during iteration 3, scenario 3, with a total of 2846 words. Its corresponding permutation is `aoweuigycntelsdrnheieoaamtfbdprlstkr', which is missing some letters (f, j, q, v, x and z), as we were expecting. Visit \href{https://richr31.github.io/building-blocks-web/}{this page} to visualize these results\\

We modified the crossover function to re include letters that a crossed child loses. After making these changes we re ran the experiment, with the same parameters as before and the results are listed in \cref{tab:genAl_results_2}.

\begin{table}[H]
    \centering
    \caption{Arguments and results for Iterations of genetic algorithm corrected.}\label{tab:genAl_results_2}
    \begin{tabular}{|l|c|c|c|c|c|c|c|}
        \cline{2-8}
        \multicolumn{1}{c|}{}& \multicolumn{7}{c|}{Iteration}\\
        \cline{2-8}
        \multicolumn{1}{c|}{}& 1 & 2 &3 &4 &5 & 6 & 7 \\
        \hline
        Generations       & 1000 & 1000 & 1000 & 500 &500 & 500 & 500\\
        \hline
        Elites            & 40   & 20   & 40   & 20  & 20   & 10 & 20\\
        Crossed type 1    & 40   & 20   & 13   & 6 & 10   & 25 & 10\\
        Crossed type 2    & 0   & 0   & 27   & 14 & 10   & 25 & 10\\
        Mutated           & 10   & 30   & 10   & 40 & 40   & 20 & 50\\
        Random            & 10   & 30   & 10   & 20 & 20   & 20 & 10\\
        \hline
        Scenario 1 best:& 2218 & 2308 & 2302 & 2302 & 2232 & 2220 & 2245\\
        Found at	:       & G991 & G158 & G658 & G256 & G151 & G171 & G397\\
        	\hline
        Scenario 2 best:& 2216 & 2224 & 2245 & 2203 & 2203 & 2219 & 2221\\
        Found at	:      & G164 & G123 & G423 & G92   & G99 & G61 & G171\\
        \hline
        Scenario 3 best:& \cellcolor{pink}2360 & \cellcolor{pink}2360 & \cellcolor{pink}2360 & \cellcolor{pink}2360 & \cellcolor{pink}2360 & \cellcolor{pink}2360 & \cellcolor{pink}2360 \\
        Found at	:      & G0 	  & G0   & G0 	& G0   & G0   & G0   & G0\\
        \hline
    \end{tabular}
\end{table}

The highest word count achieved by this `corrected' experiment was 2360, but since this occurred in every iteration of scenario 3 at generation 0, it indicates the algorithm never improved beyond an initial arbitrary individual. This would mean that the actual best word count found by the corrected version of the genetic algorithm is 2846, found in scenario 1 of iteration 2.

\section{Conclusions} \label{sec:conclusions} Our experimentation demonstrated several search and optimization algorithms for our six-cube set of `Building blocks'. \cref{tab:all_algo_comparison} shows the results found by each algorithm when we adhere to the constrain stated in \cref{sec:set_up} of including every letter of the English alphabet in our set, we see that the best word count was found by performing a greedy algorithm in \cref{subsubsec:GA}. See \cref{tab:all_algo_comparison} to compare the best word counts found by each algorithm.

\begin{table}[H]
    \centering
    \caption{Comparison of all algorithms}\label{tab:all_algo_comparison}
    \begin{tabular}{|l|l|l|}
        \hline
        Algorithm & Permutation & Word count\\
        \hline
        Random search & tpndaqaeiiuocrxleyltvhoebzrfkmjwgssn & 1881\\
        \hline
        Simulated annealing & oqieeyvbtpfdaeuiaojcrnslrnlhxskmzwgt & 2352\\
        \hline
        
        Constrained greedy  & eioxauioehjanqsrlrptgdkvbyctfsmezlnw &2268\\
        \hline
        Best First Search & yqoiaektdpvflnsjhrxaeouitegbwcrmlnzs & 2328\\
        \hline
        \cellcolor{pink}Greedy & \cellcolor{pink}oieauieoyqaetrnjlhvktdzpmslnrfgxcswb& \cellcolor{pink}2386\\
        \hline
        Genetic Corrected & ioiajeeauqeopgvtcwhkrlsxmrlzdnbytnfs & 2308\\
        \hline
    \end{tabular}
\end{table}

However, during \cref{subsec:genetic} we also found that if we relax such constrains, we can achieve a significant improvement in contrast to the other algorithms, as we saw in \cref{tab:genAl_results}. We should also bare in mind that said constraints make sense only for a set of six cubes, since the need to make sure to include every single letter of the alphabet in our set is less of an issue as we make our set larger.

With a sufficiently large set of cubes, we could potentially spell every word in our dataset. Our goal would then be to determine the smallest set of cubes that achieves this. The findings from this paper could be used to reduce our word list by eliminating the words found with our 6-cube set, enabling us to focus on maximizing coverage of the remaining words

\appendix

\section{permutation proof} \label{app:proof}
We express the 6-cube building block set as a 36-entry tuple.\\ 
$$(l_0, l_1, l_2, l_3, \dots ,l_{33}, l_{34}, l_{35}) \text{ where } l_i \text{ is a letter from the English alphabet}$$

$S_{36}$, the symmetric group on 36 elements describes all possible permutations. We will describe these elements as acting on the indices rather than the elements themselves. Additionally, we refer to a `legal' move as going from parent to child within the tree discussed in \cref{subsec:tree}.\\

We define a `legal move' as a two-cycle in $S_{36}$, (x,y) such that either one of the following holds: That i and j belong to the same cube as expressed in \cref{eq:same_cube}, or that i and j have the same color, as expressed by \cref{eq:same_color}.
\begin{equation}\label{eq:same_cube}
	6n \leq x,y < 6(n+1) \text{ for some } n \in \mathbb{N}_0
\end{equation}
\begin{equation}\label{eq:same_color}
	|x-y| \equiv 0 \quad (\mod 6)
\end{equation}

\begin{lemma}\label{lem:legal_moves}
Suppose two entries in a 6-cube building block set need to be swapped. This can be done in three of fewer `legal moves'.
\end{lemma}

\begin{proof}

Let $i $ and $j $ be the indices corresponding to the entry swap. If $(i, j) $ is a legal move, we are done.

Suppose that this is not the case.

Let $0\leq a,b\leq 5$ where $a\neq b$, and let $0\leq m,n\leq5$ where $m \neq n$.

Then, $\exists nm$ such that
$$i = 6n + a \quad \text{and} \quad j = 6m + b $$

Consider a third index $k = 6n + b $, where $k \neq i $ and $k \neq j $.

Since $0 \leq n \leq 5 $ and $0 \leq b < 6 $, it follows that 
$$ 0 \leq k \leq 6 \cdot 5 + 5 = 35$$
\textbf{Observation 1}: $(i,k)$ is a legal move since 
$$6n\leq i,k < 6(n+1)$$
\textbf{Observation 2}: $(k,j)$ is a legal move since
$$k-j = 6n+b-6m-b = 6(n-m) \equiv 0 \quad (\mod 6)$$

Hence we can permute in the following way:
$$(i,k)(k,j)(i,k)$$
This product of three moves is equivalent to $(i,j)$.

\end{proof}

\begin{theorem}
Any element in $S_{36}$ can be represented as legal moves.
\end{theorem}

\underline{Sketch of proof}: Let $\sigma \in S_{36}$. $\sigma$ can be decomposed into a product of only two cycles. By \cref{lem:legal_moves}, those two cycles can be broken  down into at most a product of three `legal moves'. Hence $\sigma$ can be written as a product of only legal moves.
Note that 36-cycle permutations can be expressed as a product of 35 two-cycles. Based on our work, a 36-cycle can be decomposed into somewhere between 35 and 105 legal moves. So, the exploration of 130 thousand legal moves, is sufficiently expressive in terms of permutations.

\section{Results of random search per maximization case}\label{app:RS}
\begin{table}[htbp]
    \centering
    \caption{Random search per maximization case}
\begin{tabular}{|c|c|c|c|c|}
        \cline{2-5}
        \multicolumn{1}{c|}{}&\textbf{Permutation} & \textbf{Mono} & \textbf{Rainbow} & \textbf{Sum} \\
        \hline
        
        \multirow{10}{*}{\rotatebox{90}{Mono maximization}}&ovrsmgbijtefnruektsezlxiwoyahadcqpln & 164 & 1405 & 1569 \\
        &qmesrezkbepdlxwlsnijgttrhnyaoocivfua & 164 & 1338 & 1502 \\
        &hsljobtifsncaeuygxewoerzrlpndkmtivaq & 159 & 1334 & 1493 \\
        &aynntismhzrdefskexralvwbtlqiocpeugjo & 158 & 1326 & 1484 \\
        &arlqymriwclnpjitghtvxsafskeobuezdnoe & 157 & 1379 & 1536 \\
        &auirdtsnxgeltkzepohqcnayrfvolbemjsiw & 157 & 1324 & 1481 \\
        &edssrohnpnwlxgrlcvymaeikjfeiuaqzttob & 157 & 1464 & 1621 \\
        &itowhalnzdetmkxeuejiyaspblrgvsoqcnfr & 156 & 1405 & 1561 \\
        &ayocrqsiwsolpziglxejhkmtrdevabtnnuef & 156 & 1293 & 1449 \\
        &tojvdsmnqiexswhallecypbtakgeuorfzrni & 155 & 1355 & 1510 \\
        \hline
        \multirow{10}{*}{\rotatebox{90}{Rainbow maximization}}&esvrlgoraxbwycutsiinomehjtklqepdafnz & 28 & 1813 & 1841 \\
        &tpndaqaeiiuocrxleyltvhoebzrfkmjwgssn & 69 & 1812 & 1881 \\
        &lwhorbnslixajtrunfecmqgpykeisdzvtaoe & 57 & 1798 & 1855 \\
        &iiyaoupxnwteargecnklmfdtoslveqjrshzb & 52 & 1790 & 1842 \\
        &nrfdvphzbneolgcmytriswajulqtkeoaesxi & 59 & 1785 & 1844 \\
        &nsvnlxdqtrsyfekioabmgtecaourjepzwlih & 68 & 1778 & 1846 \\
        &evinyrztakwiecqtumjoasehgsolfxpdrbnl & 40 & 1778 & 1818 \\
        &sfnvyhbsleloekrptdcgejmtwiauiqorxzna & 91 & 1775 & 1866 \\
        &wtkedlxlirbmnctjpnzayoherfvsegaouqis & 84 & 1773 & 1857 \\
        &trxfnlewltorkzoeasgbspqcdmnjaiuyevih & 57 & 1773 & 1830 \\
        \hline
        \multirow{10}{*}{\rotatebox{90}{Sum maximization}}&tpndaqaeiiuocrxleyltvhoebzrfkmjwgssn & 69 & 1812 & 1881 \\
        &sfnvyhbsleloekrptdcgejmtwiauiqorxzna & 91 & 1775 & 1866 \\
        &smgxtcvlthjndewioqkazoeaiyeubfrprlns & 102 & 1755 & 1857 \\
        &wtkedlxlirbmnctjpnzayoherfvsegaouqis & 84 & 1773 & 1857 \\
        &lwhorbnslixajtrunfecmqgpykeisdzvtaoe & 57 & 1798 & 1855 \\
        &jcsfwabizdaxisutoenrgklmhyeoeqplvrnt & 97 & 1753 & 1850 \\
        &wztselbgyxthoeiacunofqvrapeisjlkmdrn & 80 & 1770 & 1850 \\
        &yondtsheoriabpflxveinkautgqwcemlszjr & 84 & 1762 & 1846 \\
        &nsvnlxdqtrsyfekioabmgtecaourjepzwlih & 68 & 1778 & 1846 \\
        &nrfdvphzbneolgcmytriswajulqtkeoaesxi & 59 & 1785 & 1844 \\
        \hline
\end{tabular}
\end{table}
\clearpage
\section{Results of simulated annealing per maximization case}\\
\begin{table}[htbp]
    \centering
    \caption{Simulated annealing per maximization case}
\begin{tabular}{|c|c|c|c|c|c|}
        \cline{2-6}
        \multicolumn{1}{c|}{}&\textbf{Permutation} & \textbf{Mono} & \textbf{Rainbow} & \textbf{Sum} & \textbf{Found at} \\
        \hline
        \multirow{10}{*}{\rotatebox{90}{Mono maximization}} 
        & otmgjhwralknredxzunaiyviesefctbploqs & 227 & 1268 & 1495 & P60981\\
        & oeakqbspduvowsiyznnaexhrermfjgttllci & 226 & 1299 & 1525 & P52603\\
        & eedfkbstaxznwpecjinrmyqroailhutslovg & 225 & 1326 & 1551 & P52405\\
        & eeazkbstdxfnwsecqinriyjroamlhutplovg & 222 & 1368 & 1590 & P52321\\
        & wsilvbetdxquheegzoipmfjrnaaycnsrlokt & 221 & 1199 & 1420 & P52131\\
        & wsilvbeeaxquttegzhipmfjrnrlycnsadoko & 220 & 1239 & 1459 & P52066\\
        & rtrljcaenfqimaeubkilwgxhtpoovsesdyzn & 217 & 1207 & 1424 & P50276\\
        & ltrljcaenfqimaeubkirwgxhtpoovsesdyzn & 213 & 1217 & 1430 & P50275\\
        & etryjcaenfximaeuqkirwglhtpoovnlsdbzs & 211 & 1241 & 1452 & P50221\\
        & klorrcbasotfymtwaqnduepjlinxsviehgez & 210 & 1287 & 1497 & P47092\\
        \hline
        \multirow{10}{*}{\rotatebox{90}{Rainbow maximization}}
         & tkrosidtsqlawvjanuypnexocglzrefbmihe & 4 & 2239 & 2243 & P37893\\
         & ykrosidtsqlawvjanutpnexocglzrefbmihe & 4 & 2214 & 2218 & P37840\\
         & ykrosidtsqlatvjanuwpnexocglzrefbmihe & 4 & 2213 & 2217 & P37749\\
         & wkrosidtsqlatvjanuypnexocglzrefbmihe & 4 & 2206 & 2210 & P37653\\
         & gkrosadtsqlitvjanumynexocwlzrefbpihe & 5 & 2174 & 2179 & P37284\\
         & gkrosadtsqletvjanumynexocwlzrifbpihe & 5 & 2171 & 2176 & P37280\\
         & gcrosidtwqletvjanumynexokslarzfbpihe & 6 & 2159 & 2165 & P37074\\
         & gcrosidtwzletvjanumynexokslarqfbpihe & 6 & 2156 & 2162 & P37005\\
         & gcrosidthzletvjanumynexokslarqfbpiwe & 6 & 2154 & 2160 & P36980\\
         & gcrosidtmzletvjanuhynexokslarqfbpiwe & 6 & 2131 & 2137 & P36971\\
         \hline
        \multirow{10}{*}{\rotatebox{90}{Sum maximization}}
        &oqieeyvbtpfdaeuiaojcrnslrnlhxskmzwgt & 79 & 2273 & 2352 &P52704 \\
        &oqieeyvbtpfdaeuiaoxcrnslrnlhjskmzwgt & 76 & 2271 & 2347 &P52444\\
        &oqieeyvbtpfdaeuiaoxcrnslrnlhjskmzgwt & 72 & 2273 & 2345 &P52403\\
        &oqyeeivbtpfdaeuiaojcrnslrnlhxskmzwgt & 96 & 2240 & 2336 &P50859\\
        &oqieeywbtpfvaeuaiojmrnslrnlxhksczgdt & 73 & 2262 & 2335 &P47960\\
        &oqieeywbtpfvaeoaiuxmrlhnrglnskscjzdt & 94 & 2234 & 2328 &P46335\\
        &oqieeywbtpfvaeoaiuxmrlhnrglnsksczjdt & 91 & 2231 & 2322 &P46325\\
        &oqieeywbtpfvaeoaiuxmrlhnrglnskjczsdt & 101& 2214 & 2315 &P46180\\
        &oqieeywbtpfvaeoaiuxmrlhnrglnskdczsjt & 93 & 2220 & 2313 &P45908\\
        &oqieeywbtpfvaeoiauxmrlhnrglnskdczsjt & 92 & 2210 & 2302 &P45675\\
        \hline
\end{tabular}
\end{table}
\clearpage
\section{Results from N-ary tree algorithms}
\begin{table}[htbp]
    \centering
    \caption{Constrained greedy algorithm with the `base' permutation as root}
    \begin{tabular}{|c|c|c|c|c|l|}
        \cline{2-6}
        \multicolumn{1}{c|}{}& \textbf{Permutation} & \textbf{Mono} & \textbf{Rainbow} & \textbf{Sum} & \textbf{Found at} \\
        \hline
        \multirow{10}{*}{\rotatebox{90}{Mono Maximization}} & qreaahveoitojlnleszidbrpgfmkstxywncu & 201 & 1308 & 1509 & G484-P82115 \\
        & qreaahgeoitojlnleszidbrpxfmkstvywncu & 201 & 1299 & 1500 & G484-P82114 \\
        & qreaahzeoitojlnlesxidbrpgfmkstvywncu & 201 & 1371 & 1572 & G484-P82113 \\
        & qreaahjeoitoxlnleszidbrpgfmkstvywncu & 201 & 1350 & 1551 & G484-P82112 \\
        & vreaahxeoitojlnleszidbrpgfmkstqywncu & 201 & 1274 & 1475 & G484-P82111 \\
        & greaahxeoitojlnleszidbrpqfmkstvywncu & 201 & 1283 & 1484 & G484-P82110 \\
        & zreaahxeoitojlnlesqidbrpgfmkstvywncu & 201 & 1349 & 1550 & G484-P82109 \\
        & jreaahxeoitoqlnleszidbrpgfmkstvywncu & 201 & 1345 & 1546 & G484-P82108 \\
        & xreaahqeoitojlnleszidbrpgfmkstvywncu & 201 & 1358 & 1559 & G484-P82107 \\
        & qreaahxeoitojlnleszidbrpgfmkstvywncu & 201 & 1353 & 1554 & G483-P81912 \\
        \hline
        \multirow{10}{*}{\rotatebox{90}{Rainbow Maximization}} & eioxauioehjanqsrlrptgdkvbyctfsmezlnw & 48  & 2220 & 2268 & G25-P4444 \\
        & eioxauioehjanqsrlrptgdvkbyctfsmezlnw & 49  & 2216 & 2265 & G24-P4239 \\
        & eioxauioehjanqsrlrptgdkvbycftsmezlnw & 49  & 2211 & 2260 & G26-P4633 \\
        & eioxauioehjanqsrlrptgdvkbycftsmezlnw & 51  & 2207 & 2258 & G25-P4457 \\
        & eioxauioehjanqsrlrptgdkvbycsftmezlnw & 51  & 2206 & 2257 & G26-P4634 \\
        & eioxauioehjanqsrlrpkgdtvbyctfsmezlnw & 46  & 2206 & 2252 & G26-P4614 \\
        & eioxaujoehianqsrlrptgdkvbyctfsmezlnw & 55  & 2206 & 2261 & G26-P4582 \\
        & eioxauioehjaqnsrlrptgdvkbyctfsmezlnw & 57  & 2205 & 2262 & G23-P4207 \\
        & eioxauioehjanqsrlrptgdvkbyctfsmewlnz & 50  & 2204 & 2254 & G25-P4471 \\
        & eioxauioehjaqnsrlrptgdkvbyctfsmezlnw & 56  & 2204 & 2260 & G24-P4267 \\
        \hline
        \multirow{10}{*}{\rotatebox{90}{Sum Maximization}} & euoaaqloeitwrizslnnvcbpmhygdjtxesfkr & 55  & 2061 & 2116 & G16-P2888 \\
        & euoaaqloeitwrinslznvcbpmhygdjtxesfkr & 71  & 2044 & 2115 & G17-P3018 \\
        & euoaqaloeitwrizslnnvcbpmhygdjtxesfkr & 75  & 2039 & 2114 & G17-P2991 \\
        & euoaaqloeitwrizslnnbcvpmhygdjtxesfkr & 57  & 2054 & 2111 & G17-P3028 \\
        & euoaaqloeitwrizslnnvcmpbhygdjtxesfkr & 56  & 2054 & 2110 & G17-P3035 \\
        & euoaaqloeitwrizslnnvcbpmhygdjtxesfrk & 64  & 2041 & 2105 & G17-P3065 \\
        & euoaaqloeitwrizslnnvcbpmhygdjtxeskfr & 63  & 2042 & 2105 & G15-P2786 \\
        & euoaaqloeitwrizslnnvcmpbhygdjtxeskfr & 66  & 2038 & 2104 & G16-P2859 \\
        & quoaaeloeitwrizslnnvcbpmhygdjtxeskfr & 71  & 2033 & 2104 & G16-P2806 \\
        & euoaaqloeitwrizslnnvcbpmhygdftxesjkr & 59  & 2045 & 2104 & G15-P2710 \\
        \hline
    \end{tabular}
\end{table}

\begin{table}[htbp]
    \centering
    \caption{Constrained greedy algorithm with the `Seed2K' permutation as root}
    \begin{tabular}{|c|c|c|c|c|l|}
    		\cline{2-6}
         \multicolumn{1}{c|}{}& \textbf{Permutation} & \textbf{Mono} & \textbf{Rainbow} & \textbf{Sum} & \textbf{Found at} \\
        \hline
        \multirow{10}{*}{\rotatebox{90}{Mono Maximization}}
        & asouadttgjyrsefclbrivkieemxqwnplhzno & 203 & 1343 & 1546 & G88-P15398 \\
        & asouadrtgjyrsefclbeivkietmxqwnplhzno & 203 & 1265 & 1468 & G88-P15397 \\
        & asouadstgjyreefclbrivkietmxqwnplhzno & 203 & 1287 & 1490 & G88-P15396 \\
        & psouadetgjyrsefclbrivkietmxqwnalhzno & 203 & 1256 & 1459 & G88-P15395 \\
        & tsouadetgjyrsefclbrivkieamxqwnplhzno & 203 & 1305 & 1508 & G88-P15394 \\
        & rsouadetgjyrsefclbaivkietmxqwnplhzno & 203 & 1322 & 1525 & G88-P15393 \\
        & ssouadetgjyraefclbrivkietmxqwnplhzno & 203 & 1256 & 1459 & G88-P15392 \\
        & esouadatgjyrsefclbrivkietmxqwnplhzno & 203 & 1306 & 1509 & G88-P15391 \\
        & asouadetgjyrsefclbrikvietmxqwnplhzno & 203 & 1336 & 1539 & G88-P15356 \\
        & asouadetgjyrsefclbrivkietmxqwnplhzno & 203 & 1324 & 1527 & G87-P15209 \\
        \hline
        \multirow{10}{*}{\rotatebox{90}{Rainbow Maximization}} & qeikoirmtvwgrfhblesjapeoxsntcylnzdua & 50 & 2117 & 2167 & G18-P3303 \\
        & xeikoirmtvwgrfhblesjapeoqsntcylnzdua & 50 & 2109 & 2159 & G19-P3452 \\
        & qeikoirmtvwgrfhblesjapeonsxtcylnzdua & 42 & 2108 & 2150 & G19-P3420 \\
        & qeakoirmtvwgrfhblesjipeonsxtcylnzdua & 42 & 2106 & 2148 & G18-P3243 \\
        & qeakoirmtvwgrfhblesjipeoxsntcylnzdua & 50 & 2105 & 2155 & G17-P3143 \\
        & qeikoirmtvwgrfhblesjapeoxsntcylnudza & 58 & 2102 & 2160 & G19-P3444 \\
        & ieikoqrmtvwgrfhblesjapeoxsntcylnzdua & 52 & 2100 & 2152 & G19-P3364 \\
        & qaekoirmtvwgrfhblesjipeoxsntcylnzdua & 54 & 2099 & 2153 & G18-P3188 \\
        & qeadoirmtvwgrfhblesjipeoxsntcylnzkua & 50 & 2099 & 2149 & G16-P2866 \\
        & qeikoirmtvwgrfhblesjapeocsntxylnzdua & 49 & 2098 & 2147 & G19-P3422 \\
        \hline
        \multirow{10}{*}{\rotatebox{90}{Sum Maximization}} & qodauesytcfglhnpvrjiziaotekbwmrslnex & 65 & 2134 & 2199 & G31-P5596 \\
        & qadouesytcfglhnpvrjiziaotekbwmrslnex & 76 & 2120 & 2196 & G32-P5693 \\
        & qadouesytcfglhnpvrjiziaotekbwmrslxen & 81 & 2115 & 2196 & G31-P5516 \\
        & qokauesytcfglhnpvrjiziaotedbwmrslnex & 65 & 2130 & 2195 & G32-P5807 \\
        & jodauesytcfglhnpvrqiziaotekbwmrslnex & 65 & 2130 & 2195 & G32-P5776 \\
        & qodauesytcfglhnpvrjiziaotekbwmrslxen & 70 & 2124 & 2194 & G30-P5354 \\
        & qodauesytcfglhnpvrzijiaotekbwmrslnex & 65 & 2127 & 2192 & G32-P5732 \\
        & jodauesytcfglhnpvrqiziaotekbwmrslxen & 70 & 2121 & 2191 & G31-P5600 \\
        & qodauesytcfglhnpvrjiziaotekbwmrxlsen & 75 & 2114 & 2189 & G31-P5589 \\
        & qodauesytcfglhnpvrzijiaotekbwmrslxen & 70 & 2119 & 2189 & G31-P5555 \\
        \hline
    \end{tabular}
\end{table}

\begin{table}[htbp]
    \centering
    \caption{Constrained greedy algorithm with the best permutations from random search as roots}
    \begin{tabular}{|c|c|c|c|c|l|}
        \cline{2-6}
        \multicolumn{1}{c|}{}& \textbf{Permutation} & \textbf{Mono} & \textbf{Rainbow} & \textbf{Sum} & \textbf{Found at} \\
        \hline
        \multirow{10}{*}{\rotatebox{90}{Mono Maximization}}
        & omvsgrsiftjbtrkeuenezlxwiayaholdqpcn & 213 & 1318 & 1531 & G39-P6916 \\
        & omvsgrniftjbtrkeueiezlxwsayaholdqpcn & 213 & 1271 & 1484 & G39-P6915 \\
        & omvsgrtiftjbirkeuenezlxwsayaholdqpcn & 213 & 1238 & 1451 & G39-P6914 \\
        & lmvsgriiftjbtrkeuenezlxwsayahoodqpcn & 213 & 1260 & 1473 & G39-P6913 \\
        & smvsgriiftjbtrkeuenezlxwoayaholdqpcn & 213 & 1222 & 1435 & G39-P6912 \\
        & nmvsgriiftjbtrkeueoezlxwsayaholdqpcn & 213 & 1287 & 1500 & G39-P6911 \\
        & tmvsgriiftjborkeuenezlxwsayaholdqpcn & 213 & 1299 & 1512 & G39-P6910 \\
        & imvsgroiftjbtrkeuenezlxwsayaholdqpcn & 213 & 1280 & 1493 & G39-P6909 \\
        & omvsgriiftjbtrkeuenexlzwsayaholdqpcn & 213 & 1249 & 1462 & G39-P6875 \\
        & omvsgriiftjbtrkeuenezlxwsayaholdqpcn & 213 & 1251 & 1464 & G38-P6674 \\
        \hline
        \multirow{10}{*}{\rotatebox{90}{Rainbow Maximization}} 
        & elxrcgotavskysutbwinazemjrilqehfodnp & 27 & 2047 & 2074 & G15-P2717 \\
        & elxrcgotavqkysutbwinomezjrilsehfadnp & 23 & 2047 & 2070 & G14-P2636 \\
        & elxrcgotavskysutbwinamezjrilqehfodnp & 23 & 2047 & 2070 & G14-P2612 \\
        & elxrcgotavskysutbminazewjrilqehfodnp & 27 & 2046 & 2073 & G16-P3006 \\
        & elxrcgotavskysutbwinomezjrilqehfadnp & 23 & 2046 & 2069 & G13-P2417 \\
        & elxrcgotavqkysutbwinamezjrilsehfodnp & 23 & 2043 & 2066 & G15-P2812 \\
        & elorcgotavskysutbwinamezjrilqehfxdnp & 23 & 2042 & 2065 & G15-P2782 \\
        & elxzcgotavskysutbwinaremjrilqehfodnp & 27 & 2041 & 2068 & G16-P2969 \\
        & elxrcgotaksvysutbwinazemjrilqehfodnp & 26 & 2041 & 2067 & G16-P2863 \\
        & elxrcgotavskysutbwinamezqriljehfodnp & 20 & 2040 & 2060 & G15-P2722 \\
        \hline
        \multirow{10}{*}{\rotatebox{90}{Sum Maximization}} 
        & elxrcgotavskysutbwinazemjrilqehfodnp & 27 & 2047 & 2074 & G15-P2717 \\
        & elxrcgotavqkysutbwinomezjrilsehfadnp & 23 & 2047 & 2070 & G14-P2636 \\
        & elxrcgotavskysutbwinamezjrilqehfodnp & 23 & 2047 & 2070 & G14-P2612 \\
        & elxrcgotavskysutbminazewjrilqehfodnp & 27 & 2046 & 2073 & G16-P3006 \\
        & elxrcgotavskysutbwinomezjrilqehfadnp & 23 & 2046 & 2069 & G13-P2417 \\
        & elxrcgotavqkysutbwinamezjrilsehfodnp & 23 & 2043 & 2066 & G15-P2812 \\
        & elorcgotavskysutbwinamezjrilqehfxdnp & 23 & 2042 & 2065 & G15-P2782 \\
        & elxzcgotavskysutbwinaremjrilqehfodnp & 27 & 2041 & 2068 & G16-P2969 \\
        & elxrcgotaksvysutbwinazemjrilqehfodnp & 26 & 2041 & 2067 & G16-P2863 \\
        & elxrcgotavskysutbwinamezqriljehfodnp & 20 & 2040 & 2060 & G15-P2722 \\
        \hline
    \end{tabular}
\end{table}

\begin{table}[htbp]
    \centering
    \caption{Best-first search with the `base' permutation as root}
    \begin{tabular}{|c|c|c|c|c|l|}
    		\cline{2-6}
    		\multicolumn{1}{c|}{}& \textbf{Permutation} & \textbf{Mono} & \textbf{Rainbow} & \textbf{Sum} & \textbf{Found at} \\
        \hline
        \multirow{10}{*}{\rotatebox{90}{Mono Maximization}}
        & tnvfdpeiqoiahkxlesnsjrlrwbcgmtouzyae & 225 & 1724 & 1949 & G236-P41396
\\
& tbvfdpeiqoiansjleshnzrlrwkcgmtouxyae & 225 & 1716 & 1941 & G263-P46186
\\
& hbvfdpeiqoiankjlestnxrlrwscgmtouzyae & 225 & 1716 & 1941 & G260-P45625
\\
& tbvfdpeixoiannqleshkjrlrwscgmtouzyae & 225 & 1716 & 1941 & G258-P45305
\\
& tkvfdpeijoiahnxlesnszrlrwbcgmtouqyae & 225 & 1716 & 1941 & G250-P43936
\\
& tkvgdpeiqoiahnjlmsnsxrlrwbcfetouzyae & 225 & 1716 & 1941 & G248-P43611
\\
& tkvgdpeuqoiahnjfesnsxrlrwbclmtoizyae & 225 & 1716 & 1941 & G242-P42524
\\
& tkvgdpeixoiahnqfesnsjrlrwbclmtouzyae & 225 & 1716 & 1941 & G241-P42361
\\
& tkvgdpeizoiahnxlesnsjrlrwbcfmtouqyae & 225 & 1716 & 1941 & G240-P42188
\\
& tnvfdpeiqoiahkjgesnsxrlrwbclmtouzyae & 225 & 1716 & 1941 & G238-P41838
\\
        \hline
        \multirow{10}{*}{\rotatebox{90}{Rainbow Maximization}}
& aioeqaluecrwnyszltpfbtnvdojsrmxeighk & 54 & 2226 & 2280 & G850-P148581
\\
& aioeqaluecrwdoszltpfbtnvnyjsrmxeighk & 54 & 2207 & 2261 & G858-P150024
\\
& aioeqaluezrwnyscltpfbtnvdojsrmxeighk & 54 & 2205 & 2259 & G851-P148833
\\
& aioeqaluecrwtyszlnpfbtnvdojsrmxeighk & 54 & 2204 & 2258 & G851-P148728
\\
& aioeqawuecrmnyszltpfbtnvlojsrdxeighk & 59 & 2202 & 2261 & G855-P149418
\\
& aioexaluecrwnyszltpfbtnvdojsrmheigqk & 54 & 2202 & 2256 & G852-P149023
\\
& aioeqaluezrwnyscltpfbtnvdojsrmheigxk & 54 & 2202 & 2256 & G852-P149009
\\
& aioeqaluecrwnyszltpfbtnvmojsrdheigxk & 55 & 2201 & 2256 & G850-P148595
\\
& aioeqawuecrxnyszltvfbtnplojsrdmeighk & 53 & 2201 & 2254 & G854-P149270
\\
& aioeqaduecrwnyszltpfbtnvlojsrmheigxk & 54 & 2200 & 2254 & G852-P148966
\\
        \hline
        \multirow{10}{*}{\rotatebox{90}{Sum Maximization}}
& eioaqaiuewtxhzlmnrfesjdcpokgvtbynsrl & 71 & 2225 & 2296 & G671-P117117
\\
& qioaeaiuewtzlyxmnrbesfcdpokgvthjnsrl & 66 & 2204 & 2270 & G825-P143881
\\
& eioaaqiuewtzhylmnrbesfcdpokgvtxjnsrl & 78 & 2192 & 2270 & G820-P143053
\\
& aioaeqiuewtxbylsnrfejcsdpokgvthznmrl & 60 & 2210 & 2270 & G790-P137990
\\
& qioaeaiuewtxbylmnrfejcsdpokgvthznsrl & 69 & 2201 & 2270 & G790-P137851
\\
& eioaaxiuewtjhzlmnrqesfdcpokgvtbynsrl & 70 & 2200 & 2270 & G780-P136263
\\
& aioaqeiuewtjhzlsnrfesxdcpovgktbynmrl & 62 & 2208 & 2270 & G767-P133896
\\
& aioajeiuewtqhylmnrfesxdcpokgvtbznsrl & 57 & 2213 & 2270 & G756-P132025
\\
& eioaaqiuewtjhylmnrxesfdcpokgvtbznsrl & 74 & 2196 & 2270 & G749-P130799
\\
& eioaaqiuewtjhzlsnrfesxdcpovgktbynmrl & 81 & 2189 & 2270 & G746-P130223
\\     
        \hline
    \end{tabular}
\end{table}

\begin{table}[htbp]
    \centering
    \caption{Best-first search with the `seed2K' permutation as root}
    \begin{tabular}{|c|c|c|c|c|l|}
    		\cline{2-6}
    		\multicolumn{1}{c|}{}& \textbf{Permutation} & \textbf{Mono} & \textbf{Rainbow} & \textbf{Sum} & \textbf{Found at} \\
        \hline
        \multirow{10}{*}{\rotatebox{90}{Mono Maximization}}
        & aexoaupfvnmbsrzwlneiqeyortjcdstlkhig & 222 & 1636 & 1858 & G335-P57992
\\
& aeqoaupfvnmbsrzhlneijeiortxcdgtlkwys & 222 & 1636 & 1858 & G327-P56698
\\
& aezoaupfvnmbsrjclneiqeyortxwdstlkhig & 222 & 1636 & 1858 & G321-P55631
\\
& aeqoaupfvcmbsrznlneijeyortxwdstlkhig & 222 & 1636 & 1858 & G320-P55475
\\
& aeqoiupfvhmbsrjclneizeaortxndstlkwyg & 222 & 1636 & 1858 & G317-P54941
\\
& aeqoiupfvwmbsrzclgeijeaortxndstlkhyn & 222 & 1636 & 1858 & G315-P54644
\\
& aexoaupfvhybsrznlneiqeiortjcdstlkwmg & 222 & 1636 & 1858 & G297-P51584
\\
& aezoaupfvnmbsrxhlneiqeiortjcdstlkwyg & 222 & 1636 & 1858 & G296-P51398
\\
& eezoaupfvhmbsrxnlnaiqeiortjcdstlkwyg & 222 & 1636 & 1858 & G296-P51355
\\
& aejoaupfvhmbsrzclneiqeiortxndstlkwyg & 222 & 1636 & 1858 & G289-P50219
\\
        \hline
        \multirow{10}{*}{\rotatebox{90}{Rainbow Maximization}}
& zeokuirmbvtnrsxdlhsfagwoncetqyljipea & 34 & 2200 & 2234 & G744-P130448
\\
& ziukoelbhvtmrsxdlnsfagwoncetyqrjipea & 65 & 2170 & 2235 & G795-P139326
\\
& zeupoilbhvtmrsndlxsjageoncetqyrfikwa & 43 & 2170 & 2213 & G772-P135431
\\
& zeokuirmbvtxrsndlhsfagwoncetqyljipea & 42 & 2170 & 2212 & G760-P133221
\\
& zeokuilmbvtnrsxdlhsfagwonyetcqrjipea & 38 & 2170 & 2208 & G753-P132053
\\
& zeokuirsmvtbrnqdlhljageoncetxysfipwa & 36 & 2170 & 2206 & G850-P148975
\\
& zeokuilmbvntrsqdlhsfagwoncetxyrjipea & 36 & 2170 & 2206 & G827-P145004
\\
& zeopuilmbvtnrshdlxsfakeoncetqyrjigwa & 35 & 2170 & 2205 & G765-P134191
\\
& zeokuilmbvtnrsqdlhsgafwoncetxyrjipea & 34 & 2170 & 2204 & G829-P145313
\\
& zeokuilmbvtnrsxdlhsjagwoncetqyrfipea & 34 & 2170 & 2204 & G746-P130875
\\
        \hline
        \multirow{10}{*}{\rotatebox{90}{Sum Maximization}}
        & yqoiaektdpvflnsjhrxaeouitegbwcrmlnzs & 104 & 2224 & 2328 & G839-P146807
\\
& qxoiaektdpvflhsrnjiaeouytegbwcrmlnzs & 110 & 2207 & 2317 & G845-P147867
\\
& xqoiaedtkpvflnsjhryaeouitegbwcrmlnzs & 90 & 2227 & 2317 & G842-P147388
\\
& qxoiaedtkpvflnsjhryaeouitegbwcrmlnzs & 96 & 2220 & 2316 & G852-P149164
\\
& yqoiaextdpvflnsjhrkaeouitegbwcrmlnzs & 104 & 2212 & 2316 & G838-P146628
\\
& yqoiaedtkpvflnrshjxaeouitegbwcrmznls & 89 & 2226 & 2315 & G857-P150038
\\
& xqoiaektdpvflnsnhryaeouitegbwcrmljzs & 104 & 2211 & 2315 & G849-P148683
\\
& eqoiayktdpvflhsjnrxaeouitegbwcrmlnzs & 103 & 2212 & 2315 & G841-P147157
\\
& yxoiaedtkpvflnrshjqaeouitegbwcrmlnzs & 98 & 2216 & 2314 & G854-P149466
\\
& qeoiayktdpvflnsrhjiaeouetxgbwcrmlnzs & 110 & 2204 & 2314 & G846-P148174
\\
        \hline
    \end{tabular}
\end{table}

\begin{table}[htbp]
    \centering
    \caption{Best-first search with the best permutations from random search as roots}
    \begin{tabular}{|c|c|c|c|c|l|}
    		\cline{2-6}
    		\multicolumn{1}{c|}{}& \textbf{Permutation} & \textbf{Mono} & \textbf{Rainbow} & \textbf{Sum} & \textbf{Found at} \\
        \hline
        \multirow{10}{*}{\rotatebox{90}{Mono Maximization}}
        & wkvpdgeyjseonmxrltsfcebioiqaauhlztrn & 217 & 1650 & 1867 & G28-P4948
\\
& hfvpdgeyjteoskxrltwmcebioiqaaunlzsrn & 217 & 1650 & 1867 & G297-P50568
\\
& nmvpdteyxseoskzrlnwfcebgoiqaauhljtri & 217 & 1650 & 1867 & G235-P40262
\\
& hfjpdgemzseoskvrltwycebioiqaaunlxtrn & 217 & 1650 & 1867 & G233-P39913
\\
& hkxtdgeyjpeosmvrltwfcebioiqaaunlzsrn & 217 & 1650 & 1867 & G226-P38689
\\
& hfvpdgeyxteoskzrltwmcebioiqaaunljsrn & 217 & 1650 & 1867 & G224-P38337
\\
& hkvpdieyxseowmzrltsfcebgoiqaaunljtrn & 217 & 1650 & 1867 & G221-P37848
\\
& hkvpdgeyjseonmzrlnsfcebioiqaauwlxtrt & 217 & 1650 & 1867 & G204-P35027
\\
& hkjpdieyzseonmvrltsfcebgoiqaauwlxtrn & 217 & 1650 & 1867 & G202-P34639
\\
& hfvpdtekztlosmxregwycebioiqaaunljsrn & 217 & 1650 & 1867 & G185-P31815
\\
        \hline
        \multirow{10}{*}{\rotatebox{90}{Rainbow Maximization}}
& lnhrvmodatpkiyojeeecisbgasuzfwxrqlnt & 42 & 2216 & 2258 & G603-P106256
\\
& lnzrptotadvkixoyeeecisbgasujfwqrhlnm & 31 & 2197 & 2228 & G845-P148320
\\
& lnzrvtotadfkiyojeeecisbgasuxpwqrhlnm & 33 & 2196 & 2229 & G836-P146854
\\
& lnhrvtodatpkixoyeeecisbgasujfwqrzlnm & 32 & 2196 & 2228 & G775-P136275
\\
& lnhrvmodztpkixoyeeecisbgasujfwqralnt & 32 & 2196 & 2228 & G761-P133682
\\
& lnhrvtodztpkijoseeeciybgasuxfwqralnm & 32 & 2196 & 2228 & G757-P133087
\\
& lnhrvmodztpkijoyeeecisbgasuxfwqralnt & 32 & 2196 & 2228 & G754-P132454
\\
& lnhrnmodatfkqyojeeecisbgasuxvwirzlpt & 31 & 2196 & 2227 & G780-P137079
\\
& lnhrfmodatpkqyojeeecisbgasuxvwirzlnt & 31 & 2196 & 2227 & G762-P133919
\\
& lnhrvmosatfkiyojeeecisbgaduxpwqrzlnt & 31 & 2196 & 2227 & G750-P131821
\\
        \hline
        \multirow{10}{*}{\rotatebox{90}{Sum Maximization}}
& dpbvwtqieouacgytexrnzhiljsnlorsfmkae & 89 & 2234 & 2323 & G483-P84941
\\
& dpbvwtoieiuacgytezrhjnolxsnlqrsfmkae & 95 & 2205 & 2300 & G854-P149887
\\
& dpbvwtioeiuacgytezrfhnxlqsnlorsjmkae & 94 & 2197 & 2291 & G851-P149505
\\
& dpbvwtioeiuacgytezrfhnqlxsnlorsjmkae & 91 & 2199 & 2290 & G847-P148802
\\
& dpbvwtioieuacgytezrfhnolxsnlqrsjmkae & 75 & 2215 & 2290 & G845-P148405
\\
& dpbvwtioieqacgmtezrxhnuljsnlorsyfkae & 78 & 2212 & 2290 & G822-P144360
\\
& dpbvwtioieuacgytezrfhnolqsnljrsxmkae & 79 & 2211 & 2290 & G818-P143646
\\
& dpbvwtioieuacgytezrfhnqljsnlorsxkmae & 80 & 2210 & 2290 & G816-P143318
\\
& dpbvwtioeiuacgzteyrfhnqljsnlorsxmkae & 100 & 2190 & 2290 & G811-P142401
\\
& vpbdwtioeiuacgytezrfhnqljsnlorsxmkae & 97 & 2193 & 2290 & G811-P142363
\\        
        \hline
    \end{tabular}
\end{table}

\begin{table}[htbp]
    \centering
    \caption{Greedy algorithm with the `base' permutation as root}
    \begin{tabular}{|c|c|c|c|c|l|}
    		\cline{2-6}
    		\multicolumn{1}{c|}{}& \textbf{Permutation} & \textbf{Mono} & \textbf{Rainbow} & \textbf{Sum} & \textbf{Found at} \\
        \hline
        \multirow{10}{*}{\rotatebox{90}{Mono Maximization}}
        & qreaahveoitojlnleszidbrpgfmkstxywncu & 201 & 1308 & 1509 & G484-P82115 \\
        & qreaahgeoitojlnleszidbrpxfmkstvywncu & 201 & 1299 & 1500 & G484-P82114 \\
        & qreaahzeoitojlnlesxidbrpgfmkstvywncu & 201 & 1371 & 1572 & G484-P82113 \\
        & qreaahjeoitoxlnleszidbrpgfmkstvywncu & 201 & 1350 & 1551 & G484-P82112 \\
        & vreaahxeoitojlnleszidbrpgfmkstqywncu & 201 & 1274 & 1475 & G484-P82111 \\
        & greaahxeoitojlnleszidbrpqfmkstvywncu & 201 & 1283 & 1484 & G484-P82110 \\
        & zreaahxeoitojlnlesqidbrpgfmkstvywncu & 201 & 1349 & 1550 & G484-P82109 \\
        & jreaahxeoitoqlnleszidbrpgfmkstvywncu & 201 & 1345 & 1546 & G484-P82108 \\
        & xreaahqeoitojlnleszidbrpgfmkstvywncu & 201 & 1358 & 1559 & G484-P82107 \\
        & qreaahxeoitojlnleszidbrpgfmkstvywncu & 201 & 1353 & 1554 & G483-P81912 \\
        \hline
        \multirow{10}{*}{\rotatebox{90}{Rainbow Maximization}}
        & oieauieoyqaetrnjlhvktdzpmslnrfgxcswb & 69 & 2317 & 2386 & G652-P114941 \\
        & qaeoiueiyeoabhslxrpdgvktfncmswnljtrz & 82 & 2315 & 2397 & G380-P67131 \\
        & oieauieoyqaetrnjlhkztdvpmslnrfgxcswb & 63 & 2314 & 2377 & G650-P114550 \\
        & qaeoiueiyeoabhslxrpdgvktfncmswnjltrz & 76 & 2314 & 2390 & G381-P67261 \\
        & oieauieoyqaetrnjlhvktdzpmslnrbgxcswf & 69 & 2313 & 2382 & G653-P115247 \\
        & qaeoiueijeoabhslxrpdgvktfncmswnlytrz & 82 & 2312 & 2394 & G379-P66870 \\
        & oieauieoyqaetrnjlhkvtdzpmslnrfgxcswb & 64 & 2310 & 2374 & G651-P114769 \\
        & oieauieoyqaetrnjlhzktdvpmslnrfgxcswb & 65 & 2310 & 2375 & G651-P114762 \\
        & oieauieoyqaetrnjlhvktdzplsmnrfgxcswb & 74 & 2309 & 2383 & G653-P115129 \\
        & oieauieoyqaetrnjlhzktdvplsmnrfgxcswb & 67 & 2309 & 2376 & G652-P114954 \\
        \hline
        \multirow{10}{*}{\rotatebox{90}{Sum Maximization}}
        & euiaoahoejkxswcfgmneqdtlbiypvtrzsrln & 56 & 2230 & 2286 & G439-P77233 \\
        & auiaoexoevkhmwcfgslejdtntiypzbnqsrlr & 54 & 2225 & 2279 & G711-P125136 \\
        & auiaoevoexkhfwcmgsdejltnpiytzbrqsnlr & 54 & 2225 & 2279 & G608-P106934 \\
        & euiaoahoefkxswcjgmneqdtlbiypvtrzsrln & 56 & 2222 & 2278 & G440-P77528 \\
        & auieoahoenklswcrgxbefqpndiytvmjzsltr & 60 & 2216 & 2276 & G190-P33634 \\
        & auiaoevoexkhmwcfgslejdtntiypzbnqsrlr & 55 & 2220 & 2275 & G712-P125192 \\
        & auiaoexoefkhmwcvgslejdtntiypzbnqsrlr & 54 & 2221 & 2275 & G710-P124855 \\
        & auiaoefoezkhvwcmgsdejltnpiytxbrqsnlr & 54 & 2221 & 2275 & G656-P115438 \\
        & auiaoefoexkhvwcmgsdejltnpiytzbrqsnlr & 54 & 2221 & 2275 & G609-P107185 \\
        & auiaoexoevkhfwcmgsdejltnpiytzbrqsnlr & 55 & 2220 & 2275 & G607-P106774 \\
        \hline
    \end{tabular}
\end{table}

\begin{table}[htbp]
    \centering
    \caption{Greedy algorithm with the `Seed2K' permutation as root}
    \begin{tabular}{|c|c|c|c|c|l|}
        \cline{2-6}
        \multicolumn{1}{c|}{}&\textbf{Permutation} & \textbf{Mono} & \textbf{Rainbow} & \textbf{Sum} & \textbf{Found at} \\
        \hline
        \multirow{10}{*}{\rotatebox{90}{Mono Maximization}}
        & asouadttgjyrsefclbrivkieemxqwnplhzno & 203 & 1343 & 1546 & G88-P15398 \\
        & asouadrtgjyrsefclbeivkietmxqwnplhzno & 203 & 1265 & 1468 & G88-P15397 \\
        & asouadstgjyreefclbrivkietmxqwnplhzno & 203 & 1287 & 1490 & G88-P15396 \\
        & psouadetgjyrsefclbrivkietmxqwnalhzno & 203 & 1256 & 1459 & G88-P15395 \\
        & tsouadetgjyrsefclbrivkieamxqwnplhzno & 203 & 1305 & 1508 & G88-P15394 \\
        & rsouadetgjyrsefclbaivkietmxqwnplhzno & 203 & 1322 & 1525 & G88-P15393 \\
        & ssouadetgjyraefclbrivkietmxqwnplhzno & 203 & 1256 & 1459 & G88-P15392 \\
        & esouadatgjyrsefclbrivkietmxqwnplhzno & 203 & 1306 & 1509 & G88-P15391 \\
        & asouadetgjyrsefclbrikvietmxqwnplhzno & 203 & 1336 & 1539 & G88-P15356 \\
        & asouadetgjyrsefclbrivkietmxqwnplhzno & 203 & 1324 & 1527 & G87-P15209 \\
        \hline
        \multirow{10}{*}{\rotatebox{90}{Rainbow Maximization}}
        & svogcelzutwardahsjxtiymenpebfqrlikno & 6 & 2319 & 2325 & G593-P104536 \\
        & svogcelzutwardahsqxtiymenpebfjrlikno & 6 & 2317 & 2323 & G592-P104197 \\
        & sxogcezvutwardahsjltiymenpebfqrlikno & 6 & 2314 & 2320 & G602-P106107 \\
        & svogcexzutwardahsjltiymenpebfqrlikno & 6 & 2314 & 2320 & G594-P104636 \\
        & sxogcezvutwardahsqltiymenpebfjrlikno & 6 & 2313 & 2319 & G601-P105865 \\
        & svogcexzutwardahsqltiymenpebfqrliyno & 6 & 2313 & 2319 & G593-P104460 \\
        & svogcelzutwardahsjxtikmenpebfqrliyno & 6 & 2312 & 2318 & G594-P104686 \\
        & szogcexvutwardahsjltiymenpebfqrlikno & 6 & 2311 & 2317 & G595-P104819 \\
        & svogcelzutwardahsqxtikmenpebfjrliyno & 6 & 2311 & 2317 & G593-P104510 \\
        & szogcezvutwardahsqltiymenpebfjrlikno & 6 & 2310 & 2316 & G596-P105061 \\
        \hline
        \multirow{10}{*}{\rotatebox{90}{Sum Maximization}}
        & wozeuakycgtelsdjnhiieoaqmbfpvtrnslxr & 94 & 2266 & 2360 & G277-P48954 \\
        & wozeuakycgtelsdjnhqieoaimbfpvtrnslxr & 106 & 2253 & 2359 & G278-P49112 \\
        & wozeuakcygtelsdjnhqieoaimbtpvfrnslxr & 94 & 2262 & 2356 & G280-P49482 \\
        & wozeuagkcytelsdqnhjieoaimtfpvbrlsnxr & 94 & 2261 & 2355 & G315-P55633 \\
        & koyeuawczgtelsdjnhoieqaimbtpvfrnslxr & 90 & 2265 & 2355 & G285-P50347 \\
        & wozeuakcygtelsdjnhiieoaqmbtpvfrnslxr & 85 & 2270 & 2355 & G281-P49638 \\
        & zoweuakycgtelsdjnhiieoaqmtfpvbrnslxr & 91 & 2264 & 2355 & G277-P48889 \\
        & wozeuakycgtelsdjnhiieoaqmtfpvbrnslxr & 88 & 2266 & 2354 & G276-P48798 \\
        & kozeuawycgtelsdjnhiieoaqmtfpvbrnslxr & 88 & 2265 & 2353 & G275-P48585 \\
        & wozeuakcygtelsdqnhjieoaimbtpvfrnslxr & 93 & 2259 & 2352 & G289-P51082 \\
        \hline
    \end{tabular}
\end{table}

\begin{table}[htbp]
    \centering
    \caption{Greedy algorithm with the best permutations from random search}
    \begin{tabular}{|c|c|c|c|c|l|}
        \cline{2-6}
        \multicolumn{1}{c|}{}&\textbf{Permutation} & \textbf{Mono} & \textbf{Rainbow} & \textbf{Sum} & \textbf{Found at} \\
        \hline
        \multirow{10}{*}{\rotatebox{90}{Mono Maximization}}
        & omvsgrsiftjbtrkeuenezlxwiayaholdqpcn & 213 & 1318 & 1531 & G39-P6916 \\
        & omvsgrniftjbtrkeueiezlxwsayaholdqpcn & 213 & 1271 & 1484 & G39-P6915 \\
        & omvsgrtiftjbirkeuenezlxwsayaholdqpcn & 213 & 1238 & 1451 & G39-P6914 \\
        & lmvsgriiftjbtrkeuenezlxwsayahoodqpcn & 213 & 1260 & 1473 & G39-P6913 \\
        & smvsgriiftjbtrkeuenezlxwoayaholdqpcn & 213 & 1222 & 1435 & G39-P6912 \\
        & nmvsgriiftjbtrkeueoezlxwsayaholdqpcn & 213 & 1287 & 1500 & G39-P6911 \\
        & tmvsgriiftjborkeuenezlxwsayaholdqpcn & 213 & 1299 & 1512 & G39-P6910 \\
        & imvsgroiftjbtrkeuenezlxwsayaholdqpcn & 213 & 1280 & 1493 & G39-P6909 \\
        & omvsgriiftjbtrkeuenexlzwsayaholdqpcn & 213 & 1249 & 1462 & G39-P6875 \\
        & omvsgriiftjbtrkeuenezlxwsayaholdqpcn & 213 & 1251 & 1464 & G38-P6674 \\
        \hline
        \multirow{10}{*}{\rotatebox{90}{Rainbow Maximization}}
        & afuwvkjlernhqnostdesicytemagpxorizbl & 5 & 2305 & 2310 & G513-P90969 \\
        & anuwvkjlernhqfastdesicytemogpxorizbl & 5 & 2303 & 2308 & G509-P90286 \\
        & anuwvkjlernhqfastdemicytesogpxorizbl & 5 & 2302 & 2307 & G508-P90138 \\
        & afuwvkqlernhjnostdesicytemagpxorizbl & 5 & 2301 & 2306 & G512-P90802 \\
        & anuwvkqlernhjfastdesicytemogpxorizbl & 5 & 2298 & 2303 & G510-P90441 \\
        & anuwvkqlernhjfastdemicytesogpxorizbl & 5 & 2298 & 2303 & G509-P90265 \\
        & anuwvkxlernhqjostdesicytemagpforizbl & 9 & 2296 & 2305 & G524-P92906 \\
        & amuwvkjlernhqnostdesicytefagpxorizbl & 5 & 2295 & 2300 & G514-P91155 \\
        & anuwvkjlernhqfostdesicytemagpxorizbl & 5 & 2295 & 2300 & G510-P90474 \\
        & anuwvkjlernhqfastdemigytesocpxorizbl & 5 & 2295 & 2300 & G507-P89949 \\
        \hline
        \multirow{10}{*}{\rotatebox{90}{Sum Maximization}}
        & rlsmxniaoeuibptvwfhnlkoryzdteceqjgas & 73 & 2258 & 2331 & G666-P117592 \\
        & rlsmxniaoeuifptvwbhnlkoryzdteceqjgas & 77 & 2252 & 2329 & G667-P117800 \\
        & rlsmzniaoeuibptvwfhnlkorxydteceqjgas & 72 & 2252 & 2324 & G606-P107054 \\
        & rlsmzniaoeuibptvwfhnlkoryxdteceqjgas & 70 & 2253 & 2323 & G605-P106851 \\
        & rlsmzniaoeuifptvwbhnlkoryxdteceqjgas & 74 & 2249 & 2323 & G604-P106707 \\
        & nrsmzlaioeuipftvwbhnlkordcjteyqsxgae & 64 & 2259 & 2323 & G506-P89471 \\
        & rlsmxniaoeuibptvwfhnlkoryqdtecezjgas & 73 & 2249 & 2322 & G667-P117884 \\
        & rlsmxniaoeuibptvwfhnlkoryjdteceqzgas & 70 & 2252 & 2322 & G649-P114620 \\
        & nrsmzlaioeuipftvwbhnlkordcxteyqsjgae & 64 & 2258 & 2322 & G507-P89720 \\
        & nrsmzlaioeuipftvwbhnlkordjcdteyqsxgae & 64 & 2258 & 2322 &505-P89312 \\
        \hline
    \end{tabular}
\end{table}

\clearpage

\section*{Acknowledgments}
I would like to thank my mentor, Dr. Shahrzad Jamshidi for all the support she has given me during my time at Lake Forest College, especially in this project.

\bibliographystyle{siamplain}
\bibliography{references}

\end{document}